\documentclass[a4paper]{llncs}

\usepackage{amsfonts}
\usepackage{amsmath}
\usepackage{amssymb}
\usepackage{multirow}
\usepackage{subfig}

\usepackage{etoolbox}
\newbool{seeAll}
\usepackage[usenames,dvipsnames]{color}
\usepackage[normalem]{ulem}

\def\trackingLevel{0}

\newcommand{\changing}[1]{{\ifnumcomp{\trackingLevel}{>}{0}{\protect\color{magenta}}{}\ifnumcomp{\trackingLevel}{=}{2}{\uwave{#1}}{#1}}}

\newcommand{\adding}[1]{{\ifnumcomp{\trackingLevel}{>}{0}{\protect\color{blue}}{}\ifnumcomp{\trackingLevel}{=}{2}{\uwave{#1}}{#1}}}
\newcommand{\removing}[1]{\ifnumcomp{\trackingLevel}{=}{2}{{\protect\color{red}\sout{#1}}}{}}

\usepackage{tikz}

\usepackage[left,modulo]{lineno}
\newcommand{\prob}[1]{{\sc #1}}
\def\sat{\prob{SAT}}
\def\csp{\prob{CSP}}
\def\twosat{\prob{2-SAT}}
\def\hornsat{\prob{HornSAT}}
\def\stbackdoor{\prob{Strong $T$-Backdoor}}

\def\hset{\prob{Hitting Set}}
\def\vcover{\prob{Vertex Cover}}
\def\kclique{\prob{$k$-Clique}}

\begin{document}

\mainmatter

\title{On Backdoors To Tractable Constraint Languages\thanks{supported by ANR Project ANR-10-BLAN-0210.}}

\author{Clement Carbonnel\inst{1,3} \and Martin C. Cooper\inst{2} \and Emmanuel Hebrard\inst{1}}
\institute{CNRS, LAAS, 7 avenue du colonel Roche, F-31400 Toulouse, France\\
\email{\{carbonnel,hebrard\}@laas.fr}
\and IRIT, University of Toulouse III, 31062 Toulouse, France\\
\email{cooper@irit.fr}
\and  University of Toulouse, INP Toulouse, LAAS, F-31400 Toulouse, France
}

\maketitle

\begin{abstract}
In the context of \csp s, a strong backdoor is a subset of variables such that every complete assignment yields a residual instance guaranteed to have a specified property. If the property allows efficient solving, then a small strong backdoor provides a reasonable decomposition of the original instance into easy instances.\
An important challenge is the design of algorithms that can find quickly a small strong backdoor if one exists. We present a systematic study of the parameterized complexity of backdoor detection when the target property is a restricted type of constraint language defined by means of a family of polymorphisms. In particular, we show that under the weak assumption that the polymorphisms are idempotent, the problem is unlikely to be FPT when the parameter is either $r$ (the constraint arity) or $k$ (the size of the backdoor) unless P = NP or FPT = W[2]. When the parameter is $k+r$, however, we are able to identify large classes of languages for which the problem of finding a small backdoor is FPT.
\end{abstract}

\section{Introduction}
\label{sec:intro}

%[MOTIVATION]
Unless P=NP, the constraint satisfaction problem (\csp) is in general intractable. However, one can empirically observe that solution methods scale well beyond what a worst-case complexity analysis would suggest.

In order to explain this gap, Williams, Gomes and Selman introduced the notion of \emph{backdoor}~\cite{backdoors}. A strong backdoor is a set of variables whose complete assignments all yield an easy residual problem. When it is small, it therefore corresponds to a weak spot of the problem through which it can be attacked. Indeed, by branching first on the variables of a backdoor of size $k$, we ensure that the depth of the search tree is bounded by $k$. There exists a similar notion of weak backdoor, ensuring that at least one assignment yields an easy problem, however, we shall focus on strong backdoors and omit the adjective ``strong''.
% related notion of weak backdoor 

%[PROBLEM: THE COMPLEXITY OF FINDING BACKDOORS]
Finding small backdoors is then extremely valuable in order to efficiently solve constraint problems, however, it is very likely to be itself intractable. In order to study the computational complexity of this problem, we usually consider backdoors with respect to a given tractable class $T$, i.e., such that all residual problems fall into the class $T$.
In Boolean Satisfiability (\sat), it was shown that finding a minimum backdoor with respect to \hornsat \changing{, }\twosat\ \adding{and their disjunction }is fixed-parameter tractable with respect to the backdoor size~\cite{conf/sat/NishimuraRS04}\cite{Gaspers}. It is significantly harder, however, to do so with respect to bounded treewidth formulas~\cite{conf/focs/GaspersS13}. In this paper we study the computational complexity of finding a strong backdoor to a \emph{semantic} tractable class of \csp , in the same spirit as a very recent work by Gaspers et al.~\cite{Gaspers}. Assuming 
%(as we do throughout this paper) 
that P $\neq$ NP, semantic tractable classes are characterized by unions and intersections of languages of constraints closed by some operations. We make the following three
%four 
main contributions:

%[CONTRIBUTIONS]
\begin{itemize}

\item We first consider the case where these operations are \emph{idempotent}, and show that computing a $k$-backdoor with respect to such a class is NP-hard even on bounded arity \csp s, and W[2]-hard for the parameter $k$ if the arity is not bounded. Observe that the scope of this result is extremely wide, as most tractable \changing{classes} of interest are idempotent.

\item Then, we characterize another large category of tractable classes, that we call \emph{Helly}, and for which finding $k$-backdoors is fixed-parameter tractable in $k+r$ where $r$ is the maximum arity of a constraint. 

\item Lastly, we show that finding $k$-backdoors with respect to many semantic tractable classes that are not Helly is W[2]-hard for $k+r$ (and remains W[2]-hard for $k$ if $r$ is fixed). However, we do not \changing{prove} a strict dichotomy since a few other conditions must be met besides not being Helly.

%\item \adding{Lastly, we study the classes induced by constant operations, which are not idempotent, and prove that finding a $k$-backdoor to these classes is in $P$. This highlights the fact that there exist nondegenerate (though admittedly not very large) tractable classes for which the backdoor problem is easy.}
\end{itemize}

%[PLAN]
The paper is organized as follows: After introducing the necessary technical background in Section~\ref{sec:preli}, we study idempotent tractable classes in Section~\ref{sec:idem} and Helly classes as well as a family of non-Helly classes in Section~\ref{sec:helly}. 
%\adding{Finally, we describe in Section~\ref{sec:cons} a polynomial-time algorithm that finds the minimum-size strong backdoor to the class of constant-closed constraints.}

\section{Preliminaries}
\label{sec:preli}

\textbf{Constraint satisfaction problems} A constraint satisfaction problem (\csp) is a triplet $(X,D,C)$ where $X$ is a set of variables, $D$ is a domain of values, and $C$ is a set of constraints. For simplicity, we assume $D$ to be a finite subset of $\mathbb{N}$. A constraint is a pair $(S,R)$ where $S \subseteq X$ is the scope of the constraint and $R$ is an $|S|$-ary relation on $D$, i.e. a subset of $D^{|S|}$ representing the possible assignments to $S$. A solution is an assignment \changing{$X \rightarrow D$} that satisfies every constraint. The goal is to decide whether a solution exists.\

A \textit{constraint language} is a set of relations. The domain of a constraint language $\Gamma$ is denoted by $D(\Gamma)$ and contains all the values that appear in the tuples of the relations in $\Gamma$. Given a constraint language $\Gamma$, \csp($\Gamma$) is the restriction of the generic \csp\ to instances whose constraints are relations from $\Gamma$. The \removing{famous }\textit{Dichotomy Conjecture} by Feder and Vardi says that for every \adding{finite} $\Gamma$, \csp($\Gamma$) is either in P or NP-complete~\cite{feder1998computational}. Since the conjecture is still open, the complexity of constraint languages is a very active research area (see e.g.~\cite{Bulatov:2011:CCC:1970398.1970400}\cite{journals/ijac/BartoB13}\cite{Barto:2014:CSP:2578041.2556646}).\

It is known that the complexity of a language is determined by its set of closure operations~\cite{closure}. Specifically, an operation $f: D(\Gamma)^a \rightarrow D(\Gamma)$ of arity $a$ is a \textit{polymorphism} of $\Gamma$ if for every $R \in \Gamma$ of arity $r$ and $t_1,\ldots,t_a \in R$, $f(t_1,\ldots,t_a) = (f(t_1[1],\ldots,t_a[1]),\ldots,f(t_1[r],\ldots,t_a[r])) \in R$. A polymorphism $f$ is \textit{idempotent} if $\forall x \in D, f(x,x,\ldots,x) = x$. We denote by Pol($\Gamma$) (resp. IdPol($\Gamma$)) the set of all polymorphisms (resp. idempotent polymorphisms) of $\Gamma$. Given two languages $\Gamma_1,\Gamma_2$ with $D(\Gamma_2) \subseteq D(\Gamma_1)$, we write Pol($\Gamma_1$) $\subseteq$ Pol($\Gamma_2$) if the restriction to $D(\Gamma_2)$ of every $f \in$ Pol($\Gamma_1$) is in Pol($\Gamma_2$). A wide range of operations have been shown to induce polynomial-time solvability of any language they preserve: these include near-unanimity operations~\cite{Jeavons98consistency}, edges~\cite{idziak2010tractability}, semilattices~\cite{closure}, 2-semilattices~\cite{2sml} and totally symmetric operations of all arities~\cite{dalmau1999closure}.\\

\noindent \textbf{Composite classes} A semantic class is a set of languages. A semantic class $T$ is tractable if \csp($\Gamma$) $\in P$ for every $\Gamma \in T$, and recognizable in polynomial time if the membership problem `Does $\Gamma \in T$?' is in $P$. We say that a semantic class $T$ is \textit{atomic} if there exists an operation $f: \mathbb{N}^a \rightarrow \mathbb{N}$ such that $\Gamma \in T$ if and only if $f_{|D(\Gamma)} \in$ Pol$(\Gamma)$, where $f_{|D(\Gamma)}$ denotes the restriction of $f$ to $D(\Gamma)$. We sometimes denote such a class by $T_f$ and say that $f$ \textit{induces} $T_f$. We call a semantic class $T$ \textit{simple} if there exists a set $\mathcal{T}$ of atomic classes such that $T = \cap_{T_f \in \mathcal{T}} T_f$. Finally, a semantic class $T$ is \textit{composite} if there exists a set $\mathcal{T}$ of simple classes such that $T = \cup_{T_s \in \mathcal{T}} T_s$. In both cases, the set $\mathcal{T}$ is allowed to be infinite. Using the distributivity of intersection over union, it is easy to see that any class derived from atomic classes through any combination of intersections and unions is composite. We say that an atomic class $T_f$ is \textit{idempotent} if $f$ is idempotent. By extension, a composite class is idempotent if can be obtained by intersections and unions of idempotent atomic classes.\

\begin{example}
\label{exm:sml}
Consider the class of max-closed constraints, introduced in \cite{Jeavons95tractableconstraints}. This class is tractable as any CSP instance over a max-closed constraint language can be solved by establishing (generalised) arc-consistency. Using our terminology, this class is exactly the atomic class induced by the operation $\max(.,.)$, and thus it is composite. Max-closed constraints have been generalized to any language that admits a semilattice polymorphism, i.e. a binary operation $f$ such that $f(x,x) = x$, $f(x,y) = f(y,x)$ and $f(f(x,y),z) = f(x,f(y,z))$ for any $x,y,z \in D$~\cite{closure}. If we denote by Sml the set of all possible semilattice operations on $\mathbb{N}$, this larger class corresponds to $\cup_{f \in \text{Sml}} T_f$, which is composite but not atomic.
\end{example}

%\begin{example}
%An operation $f$ is totally symmetric if $f(x_1,...,x_a) = f(y_1,...,y_a)$ whenever $\{x_1,...,x_a\} = \{y_1,...,y_a\}$. It has been shown in \cite{symm} that \csp($\Gamma$) is solved by arc-consistency if and only if $\Gamma$ has totally symmetric polymorphisms of all arities. If we denote by TS(a) the set of all possible totally symmetric operations on $\mathbb{N}$ of arity $a$ and ATS = $\prod_{a \in \mathbb{N}^*}TS(a)$, this class can be stated as
%$$\cup_{\mathcal{F} \in \text{ATS}} \left( \cap_{f \in \mathcal{F}} T_f \right)$$
%and thus is also a composite class.
%\end{example}

\begin{example}
\label{exm:bw}
\adding{For a given language $\Gamma$, let $\overline{\Gamma}$ be the language obtained from $\Gamma$ by adding all possible unary relations over $D(\Gamma)$ with a single tuple.} Consider the \changing{very large} class $T_{BW}$ of languages $\Gamma$ such that \changing{\csp($\overline{\Gamma}$) (and thus CSP($\Gamma$))} can be solved by achieving $k$-consistency for some $k$ that only depends on $\Gamma$.
This property is equivalent to the existence of two idempotent polymorphisms $f$ and $g$ such that for every $x,y \in D(\Gamma)$~\cite{barto2014constraint}\cite{kozik2013characterizations},
%An much larger tractable class is that of languages $\Gamma$ such that \csp($\Gamma$) can be solved by establishing $k$-consistency, for some $k$ that only depends on $\Gamma$. We call this class $T_{BW}$. 
%It has been shown~\cite{BulinDJN13}\cite{ptipoly} that this property is equivalent to the existence of two \adding{idempotent} polymorphisms $f$ and $g$ such that for every $x,y \in D(\Gamma)$,
\begin{align*}
(i) & ~~ g(y,x,x,x) = g(x,y,x,x) = g(x,x,y,x) = g(x,x,x,y)\\
(ii) & ~~ f(y,x,x) = f(x,y,x) = f(x,x,y)\\
(iii) & ~~ f(x,x,y) = g(x,x,x,y)
\end{align*}
If we denote by FGBW the set of all pairs of operations $(f,g)$ on $\mathbb{N}$ satisfying these three conditions, the class $T_{BW}$ is composite and idempotent \changing{since it can be written as $T_{BW} = \cup_{(f,g) \in \text{FGBW}} (T_f \cap T_g)$.}
%and thus is composite.
\end{example}

The choice to study composite classes shows multiple advantages. First, they are general enough to capture most natural semantic tractable classes defined in the literature, and they also allow us to group together tractable languages that are solved by the same algorithm (such as arc consistency or Gaussian elimination).\removing{ However, they do not allow restrictions on the arity of the constraints.} %, and one cannot forbid the existence of some polymorphism (e.g., ``$\Gamma$ has a semilattice but is \textit{not} max-closed"). 
Second, membership in these classes is hereditary: if $\Gamma \in T$ and Pol($\Gamma$) $\subseteq$ Pol($\Gamma'$), then $\Gamma' \in T$. In particular, any sublanguage of a language in $T$ is in $T$, and every composite class contains the empty language.

\medskip

\noindent \textbf{Strong backdoors} Given an instance $(X,D,C)$ of \csp($\Gamma$), assigning a variable $x \in X$ to a value $d \in D$ is done by removing the tuples inconsistent with $x \gets d$ from the constraints whose scope include $x$, and then removing the variable $x$ from the instance (thus effectively reducing the arity of the neighbouring constraints by one). A \textit{strong backdoor} to a semantic class $T$ is a subset $S \subseteq X$ such that every complete assignment of the variables from $S$ yields an instance whose language is in $T$. Note that assigning a variable involves no further inference (e.g., arc consistency); indeed doing so has been shown to make backdoors potentially much harder to detect~\cite{bcsat}. There exist alternative forms of backdoors, such as weak backdoors~\cite{backdoors} and partition backdoors~\cite{partition}, but we only consider strong backdoors throughout this paper so we may omit the word ``strong" in proofs. The goal of this work is to study how the properties of the target semantic class $T$ affect the (parameterized) complexity of the following problem.

\medskip
\noindent \stbackdoor: Given a \csp\ instance $I$ and an integer $k$, does $I$ have a strong backdoor to $T$ of size at most $k$?
\medskip

\noindent \textbf{Parameterized complexity} A problem is \emph{parameterized} if each instance $x$ is coupled with a nonnegative integer $k$ called the \emph{parameter}. A parameterized problem is \emph{fixed-parameter tractable} (FPT) if it can be solved in time $O(f(k)|x|^{O(1)})$, where $f$ is any computable function. For instance, \vcover\ parameterized with the size $k$ of the cover is FPT as it can be solved in time $O(1.2738^k + kn)$~\cite{vcovfpt}, where $n$ is the number of vertices of the input graph. The class XP contains the parameterized problems that can be solved in time $O(f(k)|x|^{g(k)})$ for some computable functions $(f,g)$. FPT is known to be a proper subset of XP~\cite{ParameterizedComplexity}. Between these extremes lies the \removing{so-called }\emph{Weft Hierarchy}:
$$\text{FPT} = \text{W[0]} \subseteq \text{W[1]} \subseteq \text{W[2]} \subseteq \ldots \subseteq \text{XP}$$
where for every $t$, W[$t+1$] is believed to be strictly larger than W[$t$]. These classes are closed under FPT-reductions, which map an instance $(x,k)$ of a problem $L_1$ to an instance $(x',k')$ of a problem $L_2$ such that:

\begin{itemize}
\item $(x',k')$ can be built in time $O(f(k)|x|^{O(1)})$ for some computable function $f$
\item $(x',k')$ is a yes-instance if and only if $(x,k)$ is
\item $k' \leq g(k)$ for some computable function $g$
\end{itemize}

For instance, \kclique\ is W[1]-complete \cite{ParameterizedComplexity} when the parameter is $k$. Note that if considering the parameter as a constant yields an NP-hard problem, the parameterized version is not in XP (and thus not FPT) unless P = NP.

\section{General Hardness}
\label{sec:idem}

%The two parameters we will examine are 
We consider two parameters: $k$, the size of the backdoor, and $r$, the maximum arity of the constraint network. Under the very weak assumption that $T$ is composite and idempotent, we show that \stbackdoor\ is unlikely to be FPT for either of the parameters taken separately, 
%unless P = NP or FPT = W[2].
assuming that P $\neq$ NP and FPT $\neq$ W[2], as we shall do througout the paper.
In both cases, we show that our results 
%easily
extend to the class of Boolean \csp s
%.\
with minor modifications.\

Our hardness results will be obtained by reductions from various forms of the $p$-\hset\ problem: given a universe $U$, a collection $S = \{S_i \mid i=1..n\}$ of subsets of $U$ with $|S_i| = p$ and an integer $k$, does there exist a subset $H \subseteq U$ such that $|H| \leq k$ and $\forall i$, $H \cap S_i \neq \emptyset$? This problem is NP-complete for every fixed $p \geq 2$~\cite{karp}, and W[2]-complete when the parameter is $k$ and $p$ is unbounded~\cite{ParameterizedComplexity}. The special case $p=2$ is called \vcover, and the input is typically given in the form of a graph $G=(U,S)$ and an integer $k$.\

We will make use of two elementary properties of idempotent composite classes. First, any relation with a single tuple is closed by every idempotent operation. 
%\changing{If a language $\Gamma$ belongs to an indempotent class, then $\Gamma$ }
Thus, adding such a relation to a language does not affect its membership in idempotent classes. 
%\changing{In other words, if a single tuple relation belongs to all idempotent languages.}
The second property is slightly more general. Given a relation $R$ of arity $r$, let $M_R$ be the matrix whose rows are the tuples of $R$ sorted by lexicographic order (so that $M_R$ is unique). We say that a relation $R$ is an \textit{extension} of a relation $R'$ if $M_R$ has all the columns of $M_{R'}$, plus extra columns that are either constant (i.e. every value in that column is the same) or copies of some columns of $M_{R'}$. In that case, since IdPol($\{R\}$) = IdPol($\{R'\}$), $\{R\} \in T$ if and only if $\{R'\} \in T$, for every idempotent composite class $T$.\

For the rest of the document, we represent relations as lists of tuples delimited by square brackets (e.g. $R = [t_1,\ldots,t_n]$), while tuples are delimited by parentheses (e.g. $t_1 = (d_1,\ldots,d_r)$).

\subsection{Hardness on bounded arity \csp s}
\label{sct:arity}

\begin{theorem}
\label{thm:binhard}
\stbackdoor\ is NP-hard for every idempotent composite tractable class $T$, even for binary \csp s.
\end{theorem}

\begin{proof}
We reduce from \vcover. Let $I = (G,k)$ be an instance of \vcover. We consider two cases. First, suppose that $\Gamma$ $= \{[(1),(2)]$, $[(2),(3)]$, $[(1),(3)]\} \in T$. We create a \csp\ with one variable per vertex in $G$, and if two variables correspond to adjacent vertices we add the constraint $\neq_{1,2,3}$ (inequality over the $3$-element domain) between them. Since CSP($\{\neq_{1,2,3}\}$) is NP-hard and $T$ is tractable, a valid backdoor of size at most $k$ must correspond to a vertex cover on $G$. Conversely, the variables corresponding to a vertex cover form a backdoor: after every complete assignment to these variables, the language of the reduced instance is a subset of $\Gamma$, which is in $T$ since $T$ is composite and hence hereditary.\
Now, suppose that $\Gamma \notin T$. We duplicate the column of each relation in $\Gamma$ to obtain the binary language $\Gamma' = \{R_1,R_2,R_3\}$ with $R_1 = [(1,1),(2,2)]$, $R_2 = [(2,2),(3,3)]$ and $R_3 = [(1,1),(3,3)]$. Since $\Gamma'$ is an extension of $\Gamma$ and $T$ is idempotent, $\Gamma'$ is not in $T$. Then, we follow the same reduction as in the first case, except that we add the three constraints $R_1,R_2,R_3$ instead of $\neq_{1,2,3}$ between two variables associated with adjacent vertices. By construction, a backdoor must be a vertex cover. Conversely, if we have a vertex cover, after any assignment of the corresponding variables we are left with at most one tuple per constraint, and the resulting language is in $T$ by idempotency.\qed
\end{proof}

In the case of Boolean \csp s, Theorem \ref{thm:binhard} cannot apply verbatim. This is due to the fact that every binary Boolean language is a special case of \twosat\ and is therefore tractable. Thus, a binary Boolean \csp\ has always a backdoor of size 0 to any class that is large enough to contain \twosat, and the minimum backdoor problem is trivial. The next proposition shows that this is the only case for which \stbackdoor\ is not NP-hard under the idempotency condition. Note that looking for a strong backdoor in a binary Boolean \csp\ has no practical interest; however this case is considered for completeness.

\begin{proposition}
On Boolean \csp s with arity at most $r$, \stbackdoor\ is NP-hard for every idempotent composite tractable class $T$ if $r \geq 3$. For $r = 2$, \stbackdoor\ is either trivial (if every binary Boolean language is in $T$) or NP-hard.
\end{proposition}

\begin{proof}
The proof is essentially identical to that of Theorem \ref{thm:binhard}, only with a larger number of cases to examine. First, suppose that every binary Boolean language is in $T$. Then, if $r = 2$, \stbackdoor\ is trivial. If $r \geq 3$, we reduce from $3$-\hset. Let $I=(U,S,k)$ be an instance of $3$-\hset. We create a \csp\ instance with one variable per element in $U$, and for each $(u_i,u_j,u_l) \in S$, if $R_1 = [(1,0),(0,1)] \in T$ we add the constraint $R_2 = [(1,0,0),(0,1,0),(0,0,1)]$ on the corresponding variables and $R_3 = [(1,0,0),(0,1,1)]$ otherwise. Observe that, since CSP($\{R_2\}$) is known to be NP-hard (by a reduction from (1 in 3)-SAT) and $R_3$ is an extension of $R_1$, in either case the added constraint does not belong to $T$. Thus, if a backdoor of size at most $k$ exists, then the corresponding subset of $U$ must be a hitting set. Conversely, let $B$ denote the \csp\ variables associated with a hitting set of size at most $k$. If we used the constraint $R_2$ in the reduction, then after a complete assignment of $B$ the remaining constraints are a subset of $\{R_1, [(0)], [(1)],[(0,0)]\}$ which is in $T$ since $T$ is idempotent and $R_1 \in T$. If the reduction was done using $R_2$, after any assignment we are left with a \csp\ instance whose language is a subset of $\{[(0)],[(1)]\}$ which is in $T$ by idempotency. Finally, in both cases, $B$ is a backdoor of size $k$, which completes this part of the reduction. Now, suppose that there exists a binary Boolean language $\Gamma$ that is not in $T$. Once more, we reduce from \vcover. Let $(G,k)$ be an instance of \vcover\ and $\Gamma_1 = \{[(0,1)]\}$. We create a \csp\ instance with one variable per vertex in $G$. If two vertices are adjacent, we add between them the constraint of $\Gamma_1$ if $\Gamma_1 \notin T$ and all constraints in $\Gamma$ otherwise. By construction any T-backdoor of size at most $k$ must be a vertex cover, and the same line of reasoning as in the proof of Theorem \ref{thm:binhard} gives us that the variables corresponding to any hitting set of size at most $k$ is a T-backdoor, which concludes the proof.\qed
\end{proof}

\subsection{Hardness when the parameter is the size of the backdoor}
\label{sct:size}

In general, a large strong backdoor is not of great computational interest as the associated decomposition of the original instance is very impractical. Thus, it makes sense to design algorithms that are FPT when the parameter is the size of the backdoor. In this section we show that in the case of idempotent composite classes such algorithms cannot exist unless FPT = W[2]. Furthermore, we establish this result under the very restrictive condition that the input \csp\ has a single constraint, which highlights the fact that \stbackdoor\ is more than a simple pseudo-\hset\ on the constraints outside $T$.\

For any natural numbers $m,e$ ($m \geq 3$), we denote by $R_3^m(e)$ the relation obtained by duplicating the last column of $[(e+1,e,e),(e,e+1,e),(e,e,e+1)]$ until the total arity becomes $m$. It is straightforward to see that CSP($\{R_3^m(e)\}$) is NP-hard for every $m,e$ by a reduction from 1-in-3-SAT. In a similar fashion, we define $R_2^m(e)$ as an extension of $[(e+1,e),(e,e+1)]$ of arity $m$.\

\begin{theorem}
\label{thm:w21cons}
\stbackdoor\ is W[2]-hard for every idempotent composite tractable class $T$ when the parameter is the size of the backdoor, even if the \csp\ has a single constraint.
\end{theorem}

\begin{proof}
The proof is an FPT-reduction from $p$-\hset\ parameterized with solution size $k$. Let $(p,U,S)$ be an instance of $p$-\hset, where $U$ is the universe ($|U| = n$) and $S = \{S_i \mid i=1..s\}$ is the collection of $p$-sets. We assume without loss of generality that $p \geq 3$ (if this is not the case we pad each set with unique elements). We build an $n$-ary relation $R$, where each column is associated with a value from $U$, as follows. For every $S_i \in S$, we consider two cases. If $\{R_2^2(2i)\} \notin T$, we add two tuples $t_1,t_2$ to $R$ such that the restriction of $[t_1,t_2]$ to the columns corresponding to the values appearing in $S_i$ form the relation $R_2^p(2i)$, and the other columns are constant with value $2i$. If $\{R_2^2(2i)\} \in T$, we add $3$ tuples $t_1,t_2,t_3$ such that the restriction of $t_1,t_2,t_3$ to the columns corresponding to $S_i$ form $R_3^p(2i)$, and the remaining columns are constant with value $2i$. Once the relation is complete, we apply it to $n$ variables to obtain an instance of our backdoor problem. See Figure \ref{fig:kparam} for an example of the construction.

%\removing{
%\begin{figure}
%\centering
%\begin{tabular}{|ccccccc|}
%\multicolumn{1}{c}{$u_1$} & \multicolumn{1}{c}{$u_2$} & \multicolumn{1}{c}{$u_3$} & \multicolumn{1}{c}{$u_4$} & \multicolumn{1}{c}{$u_5$} & \multicolumn{1}{c}{$u_6$}  & \multicolumn{1}{c}{$u_7$}\\
%\hline
%$2$ & $2$ & $3$ & $2$ & $2$ & $2$ & $2$\\
%\hline
%$2$ & $2$ & $2$ & $3$ & $2$ & $2$ & $2$\\
%\hline
%$2$ & $2$ & $2$ & $2$ & $3$ & $2$ & $2$\\
%\hline
%$4$ & $5$ & $4$ & $4$ & $4$ & $4$ & $4$\\
%\hline
%$4$ & $4$ & $4$ & $4$ & $5$ & $5$ & $4$\\
%\hline
%$7$ & $6$ & $6$ & $6$ & $6$ & $6$ & $6$\\
%\hline
%$6$ & $6$ & $7$ & $6$ & $6$ & $6$ & $7$\\
%\hline
%\end{tabular}
%\caption{Constraint built from $U=(u_1,...,u_7)$, with $S_1=(u_3,u_4,u_5)$, $S_2=(u_2,u_5,u_6)$ and $S_3=(u_1,u_3,u_7)$. The target class $T$ contains $R^2_2(2)$, but neither $R^2_2(4)$ nor $R^2_2(6)$.}
%\label{fig:kparamold}
%\end{figure}
%}

\begin{figure}[t]
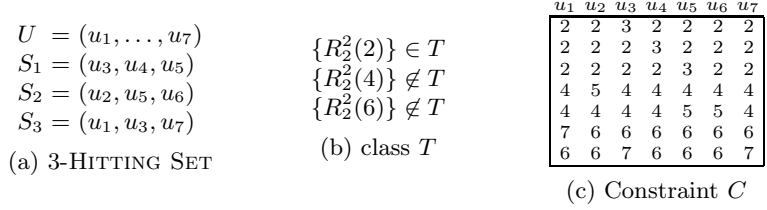

\centering
\subfloat[$3$-\hset]{
  \begin{tabular}{lll}
    %a $3$-\hset\ and a class $T$ \\
    $U $ & $ = $ & $ (u_1,\ldots,u_7)$ \\
    $S_1 $ & $ = $ & $ (u_3,u_4,u_5)$ \\
    $S_2 $ & $ = $ & $ (u_2,u_5,u_6)$ \\
    $S_3 $ & $ = $ & $ (u_1,u_3,u_7)$ \\
  \end{tabular}
}
\hspace{1cm}
\subfloat[class $T$]{
  \begin{tabular}{l}
    $\{R^2_2(2)\} \in T$ \\
    $\{R^2_2(4)\} \not\in T$ \\
    $\{R^2_2(6)\} \not\in T$
  \end{tabular}
}
\hspace{1cm}
\subfloat[Constraint $C$]{
  \scriptsize
  \begin{tabular}{|ccccccc|}
    \multicolumn{1}{c}{$u_1$} & \multicolumn{1}{c}{$u_2$} & \multicolumn{1}{c}{$u_3$} & \multicolumn{1}{c}{$u_4$} & \multicolumn{1}{c}{$u_5$} & \multicolumn{1}{c}{$u_6$}  & \multicolumn{1}{c}{$u_7$}\\
    \hline
    $2$ & $2$ & $3$ & $2$ & $2$ & $2$ & $2$\\
    %\hline
    $2$ & $2$ & $2$ & $3$ & $2$ & $2$ & $2$\\
    %\hline
    $2$ & $2$ & $2$ & $2$ & $3$ & $2$ & $2$\\
    %\hline
    $4$ & $5$ & $4$ & $4$ & $4$ & $4$ & $4$\\
    %\hline
    $4$ & $4$ & $4$ & $4$ & $5$ & $5$ & $4$\\
    %\hline
    $7$ & $6$ & $6$ & $6$ & $6$ & $6$ & $6$\\
    %\hline
    $6$ & $6$ & $7$ & $6$ & $6$ & $6$ & $7$\\
    \hline
  \end{tabular}
}
\caption{
Example of reduction from a $3$-\hset\ instance to the problem of finding a backdoor to the class $T$. 
The reduction produces a single constraint $C$.
% of a single constraint $C$ .
%Constraint built from $U=(u_1,...,u_7)$, with $S_1=(u_3,u_4,u_5)$, $S_2=(u_2,u_5,u_6)$ and $S_3=(u_1,u_3,u_7)$. The target class $T$ contains $R^2_2(2)$, but neither $R^2_2(4)$ nor $R^2_2(6)$.
}
\label{fig:kparam}
\end{figure}

Suppose we have a backdoor of size at most $k$, and suppose there exists a set $S_i$ such that none of the corresponding variables belong to the backdoor. Then, if we assign every variable in the backdoor to $2i$, the reduced constraint must belong to $T$. By idempotency, we can further assign every remaining variable outside of $S_i$ to the value $2i$ and the resulting constraint must still be in $T$. The reduced constraint becomes 
either $R_2^p(2i)$ if $\{R_2^2(2i)\} \notin T$, or $R_3^p(2i)$, which is not in $T$ since $T$ is tractable and we assume P $\neq$ NP. 
%\changing{$R_2^p(2i)$ if $\{R_2^2(2i)\} \notin T$, and $R_3^p(2i)$ otherwise.} %, since $T$ is tractable and we assume P $\neq$ NP}.
In both cases, this constraint does not belong to $T$, and we have a contradiction. Therefore, if there is a backdoor of size at most $k$, we also have a hitting set of size at most $k$.\

Conversely, suppose we have a hitting set of size at most $k$. We prove that the associated set of variables form a backdoor. Observe that two blocks (i.e. pairs/triples) of tuples of the constraint $C$ associated with different sets do not share any common value; hence, after assigning the variables corresponding to the hitting set to any values, the resulting constraint is either empty or a subrelation of a single block associated with the set $S_i$. The latter case yields two possibilities. If $T$ does not contain $\{R_2^2(2i)\}$, then the block $i$ must have been reduced to a single tuple, since the two initial tuples $t_1,t_2$ satisfy $t_1[x_j] \neq t_2[x_j]$ for all $x_j$ associated with a value in $S_i$. Thus, by idempotency, the resulting constraint is in $T$. Now, if $T$ contains $\{R_2^2(2i)\}$, the resulting constraint has at most two tuples (same argument as above), which can only happen if all the variables are assigned to the value $2i$. If we are in this situation, the new constraint must be an extension of $R_2^{2}(2i)$ and hence is in $T$. Therefore, our hitting set provides a strong backdoor in our \csp\ instance, which concludes the reduction.\qed
\end{proof}

This theorem still holds on Boolean \csp s if we allow multiple constraints in the target instance, even if these constraints are all the same relation. % (but we can force them to use the same relation).

\begin{proposition}
\label{prp:w2bool}
On Boolean \csp s, \stbackdoor\ is W[2]-hard for every idempotent composite tractable class $T$ when the parameter is the size of the backdoor, even if the \csp\ has a single type of constraint.
\end{proposition}

\begin{proof}
The proof is similar to the non-Boolean case, but more straightforward since we are allowed multiple constraints. We FPT-reduce from $p$-\hset ~\newline parameterized by solution size. Let $(p,U,S)$ be an instance of $p$-\hset, where $U$ is the universe ($|U| = n$) and $S = \{S_i \mid i=1..s\}$ is the collection of $p$-sets. We assume $p \geq 3$, as we did for Theorem \ref{thm:w21cons}. We create a CSP instance with one variable per element in $U$. For each set $S_i \in S$, if $\{R_2^p(0)\} \in T$ we add the constraint $R_3^p(0)$ on the variables corresponding to the values in $S_i$, and $R_2^p(0)$ otherwise. By construction, the language of the instance only contains constraints outside $T$, so a strong backdoor of size $k$ must intersect every constraint and hence corresponds to a hitting set of $(p,U,S)$. Conversely, the set of variables $B$ corresponding to a hitting set of size at most $k$ form a backdoor: After every assignment of $B$, at least one variable in each constraint is assigned, so the language is either formed of relations with a most one tuple (if $\{R_2^p(0)\} \notin T$) or a collection of extensions of $R_2^p(0)$ ($\{R_2^p(0)\} \notin T$) plus relations with at most one tuple. In either case, the resulting language is in $T$ by idempotency, which concludes the proof.\qed
\end{proof}

%The proof uses the same ideas as Theorem \ref{thm:w21cons} in a straightforward way, and will be omitted.\

%The proof\footnote{Available in an extended version: \url{http://homepages.laas.fr/ccarbonnel/?}} uses the same ideas as Theorem \ref{thm:w21cons} in a straightforward way. 

\begin{remark}
Partition Backdoors is an alternative form of backdoors recently introduced by Bessiere et al~\cite{partition}. Such backdoors are especially interesting in the case of conservative classes (conservativity is more restrictive than idempotency, since each polymorphism is required to satisfy $f(x_1,\ldots,x_a) \in \{x_1,\ldots,x_a\}$). The authors argue that, given a partition of the constraints $C = \{C_1,C_2\}$ such that the language of $C_1$ is in a conservative class $T$, the vertex cover of the primal graph of $C_2$ is a \removing{valid }strong backdoor to $T$. The minimum-size \emph{partition backdoor} is then the best such backdoor over every possible partition of the constraints. Computing the minimum-size partition backdoor is FPT in the parameter $k+l$, where $l$ is the size of the constraint language; our results show that computing the actual minimum strong backdoor is a much harder problem as it is still W[2]-hard for the larger parameters $k+m$ (Theorem \ref{thm:w21cons}) and $d+k+l$ (Proposition \ref{prp:w2bool}), where $m$ is the number of constraints and $d$ the size of the domain.
\end{remark}

\section{Combined parameters: Helly classes and limits}
\label{sec:helly}

We have shown in sections \ref{sct:arity} and \ref{sct:size} that considering independently the maximum constraint arity $r$ and the size of the backdoor $k$ as parameters is unlikely to yield FPT tractability. We now consider the combined parameter $k+r$ and show that FPT tractability ensues for numerous tractable composite classes.\

In order to design an algorithm for \stbackdoor\ that is FPT for $k+r$, it is important to have a procedure to \textit{check} whether a subset of variables of size at most $k$ is a strong backdoor to $T$. The natural algorithm for this task runs in time $O(mrtd^kP(\Gamma))$ (where $m$ is the number of constraints, $t$ the maximum number of tuples and $P(\Gamma)$ the complexity of the membership problem of a language $\Gamma$ in $T$) by checking independently each of the $d^k$ possible assignments of $B$. In our case this approach is not satisfactory: since $d$ is not a parameter, the term $d^k$ is problematic for the prospect of an algorithm FPT in $k+r$. The next lemma presents an alternative algorithm for the ``backdoor check" problem that is only exponential in the number of constraints $m$. Although it may seem impractical  at first sight (as $m$ is typically much larger than $k$), we will show that it can be exploited for many tractable classes.

\begin{lemma}
\label{lem:checkback}
Let $T$ be a composite class recognizable in time $P(\Gamma)$. Let $I=(X,D,C)$ be a \csp\ instance with $m$ constraints of arity at most $r$ and containing at most $t$ tuples, and $B \subseteq X$. It is possible to decide whether $B$ is a strong backdoor to $T$ in time $O(mrt^2 + m^2r(2t)^mP(\Gamma))$.
\end{lemma}

\begin{proof}
We first focus on a single constraint $(S,R)$. Let $B_S = B \cap S$. Observe that at most $t$ different assignments of $B_S$ can leave $R$ nonempty, since the subrelations of $R$ obtained with each assignment are pairwise disjoint and their union is $R$. To compute these assignments in polynomial time, one can explore a search tree. Starting from a node labelled ${R}$, we pick a nonfixed variable $v \in B_S$ and for every $d \in D(v)$ such that the subrelation $R_{v=d}$ is not empty we create a child node labelled with $R_{v=d}$. Applying this rule recursively, we obtain a tree of depth at most $r$ and with no more than $t$ leaves, so it has at most $rt$ nodes. The  time spent at each node is $O(t)$, so computing all leaves can be done in time $O(rt^2)$.\
Now, suppose that for each constraint $c = (R,S) \in C$ we have computed this set $\phi_c$ of all the locally consistent assignments of $B_S$ and stored the resulting subrelation. For every $\phi \in \prod_{c \in C} \phi_c$ and every possible subset $C'$ of the constraints, we check if the restriction $\phi_s$ of $\phi$ to the constraints of $C'$ is a consistent assignment (i.e., no variable is assigned multiple values). If $\phi_s$ is consistent, we temporarily remove from the instance the constraints outside $C'$, we apply the assignment $\phi_s$ and we check whether the language of the resulting instance is in $T$. The algorithm returns that $B$ is a backdoor if and only if each membership test in $T$ is successful.\
 To prove the correctness of the algorithm, suppose that $\psi$ is an assignment of $B$ such that the resulting language is not in $T$. Then, at least one subset of the constraints have degraded into non-empty subrelations. For each of these constraints, the restricted assignment $\psi_R$ is consistent with the others, so the algorithm must have checked membership of the resulting language in $T$ and concluded that $B$ is not a strong backdoor. Conversely, if $B$ is a strong backdoor, every complete assignment of $B$ yields an instance in $T$. In particular, if we consider only a subset of the constraints after each assignment, the language obtained is also in $T$ since $T$ is composite (and hence hereditary). Thus, none of the membership tests performed by the algorithm will fail. \changing{The complexity of the algorithm is $O(mrt^2 + m^2rt^m2^mP(\Gamma))$.}\qed
\end{proof}

We say that a composite class $T$ is $h$-Helly if it holds that for any language $\Gamma$, if every $\Gamma_h \subseteq \Gamma$ of size at most $h$ is in $T$ then $\Gamma$ is in $T$. This property is analogous to the well-studied Helly properties for set systems. We call Helly number of $T$ the minimum positive integer $h$ such that $T$ is $h$-Helly. Being characteristic of a class defined exclusively in terms of polymorphisms over $\mathbb{N}$, the Helly number is independant from parameters like the domain size or the arity of the languages. The next theorem is the motivation for the study of such classes, and is the main result of this section.

\begin{theorem}
\label{thm:fpthelly}
For every fixed composite class $T$ recognizable in polynomial time, if $T$ has a finite Helly number then \stbackdoor\ is FPT when the parameter is $k+r$, where $k$ is the size of the backdoor and $r$ is the maximum constraint arity.
\end{theorem}

\begin{proof}
Let $h$ denote the Helly number of $T$. The algorithm is a bounded search tree that proceeds as follows. Each node is labelled by a subset of variables $B$. The root of the tree is labelled with the empty set. At each node, we examine every possible combination of $h$ constraints and check if $B$ is a strong backdoor for the subset in time $O(hrt^2 + h^2r(2t)^hP(\Gamma))$ (where $P(\Gamma)$ is the polynomial complexity of deciding the membership of a language $\Gamma$ in $T$) using Lemma \ref{lem:checkback}. Suppose that $B$ is a strong backdoor for every $h$-subset. Then, for any possible assignment of $B$, each $h$-subset of the constraints of the resulting instance must be in $T$: otherwise, $B$ would not be a strong backdoor for the $h$ original constraints that generated them. Since $T$ is $h$-Helly, we can conclude that $B$ is a valid strong backdoor for the whole instance. Now suppose that we have found a $h$-subset for which $B$ is not a strong backdoor. For every variable $x$ in the union of the scopes of the constraints in this subset that is not already in $B$ (there are at most $rh$ such variables $x$), we create a child node labelled with $B \cup \{ x \}$. At each step we are guaranteed to add at least one variable to $B$, so we stop creating child nodes when we reach depth $k$. The algorithm returns `YES' at the first node visited that corresponds to a strong backdoor, and `NO' if no such node is found.\

If no strong backdoor of size at most $k$ exists, it is clear that the algorithm correctly returns `NO'. Now suppose that a strong backdoor $\mathbf{B}$, $|\mathbf{B}| \leq k$, exists. Observe that if a node is labelled with $B \subset \mathbf{B}$ and $B$ is not a backdoor for some $h$-subset of constraints, then $\mathbf{B}$ contains at least one more variable within this subset. Since the algorithm creates one child per variable that can be added and the root is labelled with a subset of $\mathbf{B}$, by induction there must be a path from the root to a node labelled with $\mathbf{B}$ and the algorithm returns `YES'.\

The complexity of the procedure is $O\left((rh)^km^h(hrt^2 + h^2r(2t)^hP(\Gamma))\right) = O\left(f(k+r)m^h(ht^2 + h^2(2t)^hP(\Gamma))\right)$.\qed
\end{proof}

In contrast to the previous hardness results, the target tractable class is not required to be idempotent. However, the target class must have a finite Helly number, which may seem restrictive. The following series of results aims to identify composite classes with this particular property.

\begin{lemma}
\label{lem:1helly}
A composite class $T$ is simple if and only if it is $1$-Helly. 
\end{lemma}

\begin{proof}
Let $T$ be a simple class, i.e. an intersection of atomic classes $T = \cap_{f \in \mathcal{F}}T_f$. Let $\Gamma$ be a constraint language such that each $\{R\} \subseteq \Gamma$ is in $T$. Then, every $f \in \mathcal{F}$ preserves every relation in $\Gamma$ and thus preserves $\Gamma$, so $\Gamma \in T$ and $T$ is $1$-Helly.\
Conversely, let $T$ be a $1$-Helly composite class. Let $\mathcal{F} = \{f \mid f$ preserves every $\{R\} \in T \}$. Every $\Gamma \in T$ admits as polymorphism every $f \in \mathcal{F}$ (as each $\{R\} \subseteq \Gamma$ is in $T$ and thus is preserved by $f$), so $T \subseteq \cap_{f \in \mathcal{F}}T_f$. The other way round, a language $\Gamma$ in $\cap_{f \in \mathcal{F}}T_f$ is preserved by every $f \in \mathcal{F}$ and thus must be a sublanguage of $\Gamma_{\infty} = \cup_{\{R\} \in T}\{R\}$, which is in $T$ since $T$ is $1$-Helly, so $\Gamma \in T$ and $\cap_{f \in \mathcal{F}}T_f \subseteq T$. Finally, $T = \cap_{f \in \mathcal{F}}T_f$ and so $T$ is simple.\qed
\end{proof}

\begin{proposition}
\label{prp:helly}
Let $h$ be a positive integer and $\mathcal{T}$ be a set of simple classes. Then, $T = \{ \Gamma \mid \Gamma$ belongs to every $T_i \in \mathcal{T}$ except at most $h\}$ is a $(h+1)$-Helly composite class.
\end{proposition}

\begin{proof}
%T is composite since it is the union of every possible intersection of all but $h$ classes from $\mathcal{T}$ and any class derived from atomic classes through any combination of intersections and unions is composite. We write $\mathcal{T} = \{T_i \mid i \in I\}$. Let $\Gamma$ be a language such that every sublanguage of size $h+1$ is in $T$. For each $R \in \Gamma$ we define $S(R) = \{T_i \mid \{R\} \notin T_i\}$. By Lemma \ref{lem:1helly}, simple classes are $1$-Helly so $\Gamma \notin T_i$ if and only if $T_i \in \cup_{R \in \Gamma}S(R)$ and  then $\Gamma \in T$ if and only if $|\cup_{R \in \Gamma}S(R)| \leq h$. We discard from $\Gamma$ every relation $R$ such that $|S(R)| = 0$ as they have no influence on the membership of $\Gamma$ in $T$. If that process leaves $\Gamma$ empty, then it belongs to $T$. Otherwise, let $s_i$ denote the maximum size of the union of $i$ sets $S(R)$. Since each sublanguage of size $i$ is in $T$, we have just shown that $1 \leq s_1 \leq .. \leq s_{h+1} \leq h$, thus there exists $j<h+1$ such that $s_j = s_{j+1}$. Let $S_j$ denote a set of $j$ relations such that $|\cup_{R \in S_j}S(R)| = s_j$. Suppose there exists $R_i \in \Gamma$ such that $S(R_i) \not\subseteq \cup_{R \in S_j}S(R)$. Then, $|\cup_{R \in S_j \cup S(R_i)} S(R)| > s_j = s_{j+1}$, and we get a contradiction. Finally, $\cup_{R \in \Gamma}S(R) \subseteq \cup_{R \in S_j}S(R)$, thus $|\cup_{R \in \Gamma}S(R)| \leq h$ and $\Gamma$ is in $T$. Therefore, $T$ is $(h+1)$-Helly.\qed
T is composite since it is the union of every possible intersection
of all but $h$ classes from $\mathcal{T}$ and any class derived from
atomic classes through any combination of intersections and unions is
composite. We write $\mathcal{T} = \{T_i \mid i \in I\}$. Let
$\Gamma$ be a language such that every sublanguage of size at most
$h+1$ is in $T$. For each $R \in \Gamma$ we define $S(R) = \{T_i \mid
\{R\} \notin T_i\}$. By Lemma \ref{lem:1helly}, simple classes are
$1$-Helly so $\Gamma \notin T_i$ $\Leftrightarrow$ ($\exists R \in
\Gamma$ such that $\{R\} \notin T_i$)  $\Leftrightarrow$ $T_i \in
\cup_{R \in \Gamma}S(R)$. So $\Gamma \in T$ if and only if $|\cup_{R
\in \Gamma}S(R)| \leq h$. We discard from $\Gamma$ every relation $R$ such that $|S(R)| = 0$ as
they have no influence on the membership of $\Gamma$ in $T$. If that
process leaves $\Gamma$ empty, then it belongs to $T$. Otherwise, let $s_j$ denote the maximum size of $\cup_{R \in \Gamma_j} S(R)$
over all size-$j$ subsets $\Gamma_j$ of $\Gamma$. Since each
sublanguage $\Gamma_j$ of size $j \leq h+1$ is in $T$, from the
argument above we have $1 \leq s_1 \leq \ldots \leq s_{h+1} \leq h$, thus
there exists $j<h+1$ such that $s_j = s_{j+1}$. Let $\Gamma_j
\subseteq \Gamma$ denote a set of $j$ relations such that $|\cup_{R
\in \Gamma_j}S(R)| = s_j$. Suppose there exists $R_0 \in \Gamma$ such
that $S(R_0) \not\subseteq \cup_{R \in \Gamma_j}S(R)$. Then,
$|\cup_{R \in \Gamma_j \cup \{R_0\}} S(R)| > s_j = s_{j+1}$, and we get a contradiction. So
$\cup_{R \in \Gamma}S(R) \subseteq \cup_{R \in \Gamma_j}S(R)$, hence
$|\cup_{R \in \Gamma}S(R)| \leq h$ and $\Gamma$ is in $T$. Therefore,
$T$ is $(h+1)$-Helly.\qed
\end{proof}

In the particular case where $\mathcal{T}$ is finite and $h=|\mathcal{T}|-1$, we get the following nice corollary. Recall that a composite class is any union of simple classes.

\begin{corollary}
\label{cor:helly}
Any union of $h$ simple classes is $h$-Helly.
\end{corollary}

\begin{example}
\label{exm:helly}
Let $T = \{ \Gamma \mid \Gamma$ is either min-closed, max-closed or 0/1/all$\}$. $T$ is the union of $3$ well-known tractable semantic classes. By definition, min-closed and max-closed constraints are respectively the languages that admit $\min(.,.)$ and $\max(.,.)$ as polymorphisms. Likewise, 0/1/all constraints have been shown to be exactly the languages that admit as polymorphism the majority operation~\cite{closure}\
$$
f(x,y,z)=\left\{
\begin{array}{c l}     
    y & \text{if } y = z\\
    x & \text{otherwise}
\end{array}\right.
$$
Thus, $T$ is the union of 3 atomic classes and hence is $3$-Helly by Corollary \ref{cor:helly}.
% and by Corollary \ref{cor:helly} it is $3$-Helly. 
Since $T$ is also recognizable in polynomial time, by Theorem \ref{thm:fpthelly} \stbackdoor\ is FPT 
%with respect to 
when parameterized by
backdoor size and maximum arity.
\end{example}

In the light of these results, it would be very interesting to show a dichotomy. Is \stbackdoor\ with parameter $k+r$ at least W[1]-hard for every tractable composite class $T$ that does not have a finite Helly number? While we leave most of this question unanswered, we have identified generic sufficient conditions for W[2]-hardness when $r$ is fixed and the parameter is $k$.\

Given a bijection $\phi: D_1 \rightarrow D_2$, we denote by $R_{\phi}$ the relation $[(d,\phi(d)), d \in D_1]$. Given a language $\Gamma$, a subdomain $D_1$ of $D(\Gamma)$ is said to be \textit{conservative} if every $f \in$ Pol($\Gamma$) satisfies $f(x_1,\ldots,x_m) \in D_1$ whenever $\{x_1,\ldots,x_m\} \subseteq D_1$. For instance, $D(\Gamma')$ is conservative for every $\Gamma' \subseteq \Gamma$, and for every column of some $R \in \Gamma$ the set of values that appear in that column is conservative. Then, we say that a class $T$ is \textit{value-renamable} if for every $\Gamma \in T$ and $\phi : D_1 \rightarrow D_2$, where $D_1$ is a conservative subdomain of $\Gamma$ and $D_2 \cap D(\Gamma) = \emptyset$, $\Gamma \cup \{R_\phi\}$ is in $T$. For instance, the class of 0/1/all constraints introduced in Example \ref{exm:helly} is value-renamable, but max-closed constraints are not (as they rely on a fixed order on $\mathbb{N}$). We also say that a composite class $T$ is \textit{domain-decomposable} if for each pair of languages $\Gamma_1 \in T$ and $\Gamma_2 \in T$, $D(\Gamma_1) \cap D(\Gamma_2) = \emptyset$ implies $\Gamma_1 \cup \Gamma_2 \in T$. Value-renamability and domain-decomposability are natural properties of any class that is large enough to be invariant under minor (from the algorithmic viewpoint) modifications of the constraint languages.\

Given a language $\Gamma$ and a bijection $\phi: D(\Gamma) \rightarrow D'$, we denote by $\phi(\Gamma)$ the language over $D'$ obtained by replacing every tuple $t = (d_1,\ldots,d_r)$ in every relation in $\Gamma$ by $\phi(t) = (\phi(d_1),\ldots,\phi(d_r))$.

\begin{lemma}
\label{lem:valueren}
Let $\Gamma = \Gamma_1 \cup \Gamma_2 \cup \{R_\phi\}$ where $\phi$ is a bijection from $D(\Gamma_2)$ to some domain $D_1$. Then, Pol($\Gamma$) $\subseteq$ Pol($\Gamma_1 \cup \phi(\Gamma_2)$).
\end{lemma}

\begin{proof}
Let $f \in $ Pol($\Gamma$) of arity $a$. We only need to show that $f$ preserves $\phi(\Gamma_2)$, as $f$ already preserves $\Gamma_1$. Since $f$ preserves $R_\phi$, for each $(d_1,\phi(d_1)),\ldots,(d_a,\phi(d_a)) \in R_\phi$ we have $(f(d_1,\ldots,d_a), f(\phi(d_1),\ldots,\phi(d_a))) \in R_\phi$, so $f(\phi(d_1),\ldots,\phi(d_a)) = \phi(f(d_1,\ldots,d_a))$ for every $d_1,\ldots,d_a \in D(\Gamma_2)$. Then, given $a$ tuples $\phi(t_1),\ldots,\phi(t_a)$ of $\phi(\Gamma_2)$, $f(\phi(t_1),\ldots,\phi(t_a)) = \phi(f(t_1,\ldots,t_a)) \in \phi(\Gamma_2)$ since $f(t_1,\ldots,t_a) \in \Gamma_2$. Therefore, $f$ is a polymorphism of $\phi(\Gamma_2)$ and Pol($\Gamma$) $\subseteq$ Pol($\Gamma_1 \cup \phi(\Gamma_2)$).
\end{proof}

\begin{theorem}
\label{thm:nothelly}
On \csp s with arity at most $r$, if $T$ is a composite class that is
\begin{itemize}
\item idempotent
\item not $1$-Helly for constraints of arity at most $r$
\item value-renamable
\item domain-decomposable
\end{itemize}
then \stbackdoor\ is W[2]-hard when the parameter is $k$.
\end{theorem}

\begin{proof}
%Since $T$ is not $1$-Helly, there exists a minimal language (obtained by minimizing successively the arity, the domain size and the number of relations, in that order) $\Gamma_m = \{R_i \mid i \in 1..l_m\}$ over a domain $D_m$ such that every sublanguage is in $T$ but $\Gamma_m$ is not. $T$ is fixed, so $\Gamma_m$ has constant size. Its arity and domain size will be denoted by $r_m$ and $d_m$. By assumption, $r_m \leq r$ and $l_m > 1$.
Since $T$ is not $1$-Helly for constraints of arity at most $r$, there exists a language $\Gamma_m = \{R_i \mid i \in 1..l_m\}$ (of arity $r_m \leq r$ and over a domain $D_m, |D_m| = d_m$) such that $l_m > 1$ and every sublanguage of $\Gamma_m$ is in $T$ but $\Gamma_m$ is not. Since $T$ is fixed, we shall consider that $\Gamma_m$ is fixed as well and hence has constant size. We assume for simplicity of presentation that \changing{every $R \in \Gamma_m$ has arity $r_m$.}

%the arities of all relations in $\Gamma_m$ are all $r_m$.

We perform an FPT-reduction from $p$-\hset\ parameterized with solution size as follows. Let $(p,U,S)$ be an instance of $p$-\hset, with $S = \{S_1,\ldots,S_s\}$ and $U = \{u_1,\ldots,u_n\}$. For every $u_i \in U$, we associate a unique variable $x_i$. For every set $S_j = (u_{\sigma_j(1)},\ldots,u_{\sigma_j(p)})$, we add $2r_m$ new variables $y^1_j,\ldots,y^{r_m}_j,z^1_j,\ldots,z^{r_m}_j$ and we create $p+2$ new disjoint domains $D_j^i$, $i \in [0 \ldots p+1]$ of size $d_m$. Then, we pick a chain of $p+1$ bijections $\psi_j^i: D_j^{i} \rightarrow D_j^{i+1}$, $i \in [0 \ldots p]$ and we add a chain of constraints $R_{\psi_j^i}$ between the $p+2$ variables $(y^{r_m}_j,x_{\sigma_j(1)},\ldots,x_{\sigma_j(p)},z^1_j)$. Afterwards, we pick a bijection $\phi_j: D_m \rightarrow D_j^{0}$ and we apply $\phi_j(R_1)$ to $y^1_j,\ldots,y^{r_m}_j$. In the same fashion, if we denote by $\psi_j$ the bijection from $D_j^{0}$ to $D_j^{p+1}$ obtained by composition of all the $\psi_j^i$, we apply every constraint in $(\psi_j \circ \phi_j)(\Gamma_m \backslash \{R_1\})$ to the variables $z^1_j,\ldots,z^{r_m}_j$. The main idea behind the construction is that both $\Gamma_m \backslash \{R_1\}$ and $\{R_1\}$ are in $T$ but $\Gamma_m$ is not: by adding $\phi_j(R_1)$ on the variables $y$, $\psi_j \circ \phi_j(\Gamma_m \backslash \{R_1\})$ on the variables $z$ and the chain of bijections $R_{\psi_j^i}$ of the $x$ variables, we have a language that is not in $T$ but assigning any value to $x$ yields a residual language in $T$ (the proof can be found below). We use this property to encode a \hset\ instance. See Figure \ref{fig:nothelly} for an example of the reduction.

\begin{figure}[t]

\centering
\begin{tikzpicture}

\node[draw, circle] at (0,0) {$x_1$};
\node[draw, circle] at (1,0) {$x_2$};
\node[draw, circle] at (2,0) {$x_3$};
\node[draw, circle] at (3,0) {$x_4$};
\node[draw, circle] at (4,0) {$x_5$};
\node[draw, circle] at (5,0) {$x_6$};
\node[draw, circle] at (6,0) {$x_7$};

%S1
\node[draw, circle] at (-2,2) {$y^1_1$};
\node[draw, circle] at (-1,2) {$y^2_1$};
\draw[<-, >=latex] (-1,2.4) arc (30:150:0.55cm);
\draw[->, >=latex] (-0.7,1.7) -- (0.9,0.35);
\draw[<-, >=latex] (2.9,0.35) arc (52:128:1.4cm);
\draw[<-, >=latex] (3.95,0.35) arc (30:150:0.55cm);
\node[draw, circle] at (7,2) {$z^1_1$};
\node[draw, circle] at (8,2) {$z^2_1$};
\draw[->, >=latex] (4.1,0.35) -- (6.7,1.7);
\draw[<-, >=latex] (8,2.4) arc (30:150:0.55cm);
\draw[<-, >=latex] (8,1.6) arc (-30:-150:0.55cm);

%Constraints labels
\node[] at (-1.5, 3) {$\phi_1(R_1)$};
\node[] at (0.2, 1.4) {$R_{\psi_1^{0}}$};
\node[] at (2, 0.9) {$R_{\psi_1^1}$};
\node[] at (3.5, 0.9) {$R_{\psi_1^2}$};
\node[] at (5.3, 1.4) {$R_{\psi_1^{3}}$};
\node[] at (7.5, 3) {$(\psi_1 \circ \phi_1)(R_2)$};
\node[] at (7.5, 1) {$(\psi_1 \circ \phi_1)(R_3)$};

%S2
\node[draw, circle] at (-2,-2) {$y^1_2$};
\node[draw, circle] at (-1,-2) {$y^2_2$};
\draw[<-, >=latex] (-1,-1.6) arc (30:150:0.55cm);
\draw[->, >=latex] (-0.7,-1.7) -- (-0.1,-0.35);
\draw[<-, >=latex] (2.9,-0.35) arc (-60:-120:2.7cm);
\draw[<-, >=latex] (4.9,-0.35) arc (-52:-128:1.4cm);
\node[draw, circle] at (7,-2) {$z^1_2$};
\node[draw, circle] at (8,-2) {$z^2_2$};
\draw[->, >=latex] (5.1,-0.35) -- (6.7,-1.7);
\draw[<-, >=latex] (8,-2.4) arc (-30:-150:0.55cm);
\draw[<-, >=latex] (8,-1.6) arc (30:150:0.55cm);

%Constraints labels
\node[] at (-1.5, -1) {$\phi_2(R_1)$};
\node[] at (0, -1.3) {$R_{\psi_2^{0}}$};
\node[] at (1.5, -1.1) {$R_{\psi_2^1}$};
\node[] at (4, -1.1) {$R_{\psi_2^2}$};
\node[] at (5.8, -1.7) {$R_{\psi_2^{3}}$};
\node[] at (7.5, -3) {$(\psi_2 \circ \phi_2)(R_3)$};
\node[] at (7.5, -1) {$(\psi_2 \circ \phi_2)(R_2)$};

\end{tikzpicture}
\caption{Example of the construction for $U=(u_1,\ldots,u_7)$, two sets $S_1=(u_2,u_4,u_5)$, $S_2 = (u_1,u_4,u_6)$, $\Gamma_m = \{R_1,R_2,R_3\}$ and $r_m = 2$. Each arrow is a (binary) constraint. The upper part of the instance is constructed from $S_1$ and the lower part from $S_2$.}
\label{fig:nothelly}
\end{figure}
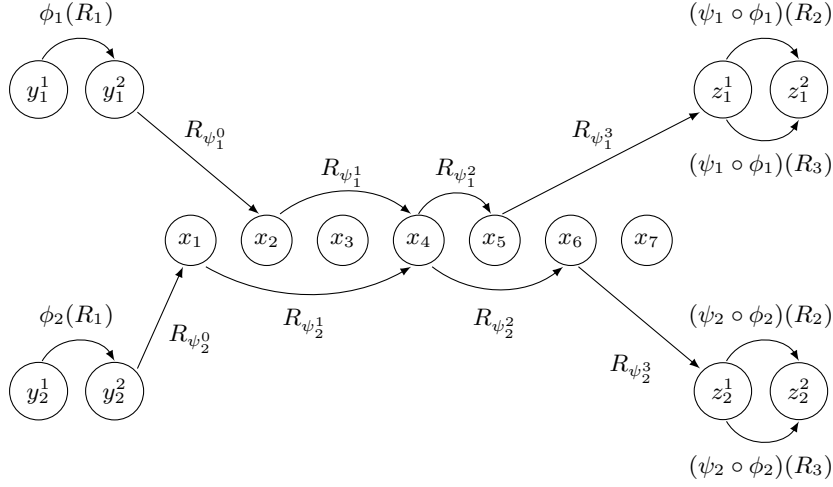

Suppose we have a backdoor to $T$ of size at most $k$. Then, for each set $S_j$, at least one variable from $(y^1_j,\ldots,y^{r_m}_j,x_{\sigma_j(1)},\ldots,x_{\sigma_j(p)},z^1_j,\ldots,z^{r_m}_j)$ must belong to the backdoor. Suppose this is not the case. Then, the language $\Gamma$ of any reduced instance would contain the relations of $\{\phi_j(R_1), (\psi_j \circ \phi_j)(\Gamma_m \backslash \{R_1\})\}$ plus the relations $R_{\psi_j^i}$. Applying Lemma \ref{lem:valueren} $p+1$ times, we get Pol($\Gamma$) $\subseteq$ Pol($(\psi_j \circ \phi_j)(R_1) \cup (\psi_j \circ \phi_j) (\Gamma_m \backslash \{R_1\})$) = Pol($(\psi_j \circ \phi_j)(\Gamma_m)$). Thus, if $\Gamma$ is in $T$, then so is $(\psi_j \circ \phi_j)(\Gamma_m)$. Then, by value-renamability $\{(\psi_j \circ \phi_j)(\Gamma_m) \cup R_{(\psi_j \circ \phi_j)^{-1}}\}$ is also in $T$ and using Lemma \ref{lem:valueren} again, $\Gamma_m$ is in $T$, which is a contradiction. Therefore, a hitting set of size at most $k$ can be constructed by including every value $u_i$ such that $x_i$ is in the backdoor, and if any variable from $y^1_j,\ldots,y^{r_m}_j,z^1_j,\ldots,z^{r_m}_j$ belongs to the backdoor for some $j$, we also include $u_{\sigma_j(1)}$.

Conversely, a hitting set forms a backdoor. After every complete assignment of the variables from the hitting set, the set of constraints associated with any set $S_j$ can be partitioned into sublanguages whose domains have an empty intersection(see Figure \ref{fig:nothelly}). The sublanguages are either:
\begin{itemize}
\item $\phi_j(R_1)$ together with some constraints $R_{\psi_j^i}$ and a residual unary constraint with a single tuple. This language is in $T$ by Lemma \ref{lem:valueren}, value-renamability and idempotency.
\item $(\psi_j \circ \phi_j)(\Gamma_m \backslash \{R_1\})$ together with some constraints $R_{\psi_j^i}$ and a residual unary constraint with a single tuple. This case is symmetric.
\item A (possibly empty) chain of constraints $R_{\psi_j^i}$ plus unary constraints with a single tuple, which is again in $T$ since $T$ is idempotent, value-renamable and contains the language $\{\emptyset\}$.
\end{itemize}
Furthermore, the sublanguages associated with different sets $S_j$ also have an empty domain intersection. Since $T$ is domain-decomposable, the resulting language is in $T$.\qed
\end{proof}

Note that this result does not conflict with Theorem \ref{thm:fpthelly}, since any class $T$ that is domain-decomposable, value-renamable and not $1$-Helly cannot have a finite Helly number (part of the proof of Theorem \ref{thm:nothelly} amounts to showing that one can build in polynomial time arbitrarily large languages $\Gamma$ such that every sublanguage is in $T$ but $\Gamma$ is not). While the proof may seem technical, the theorem is actually easy to use and applies almost immediately to many known tractable classes: to prove W[2]-hardness of \stbackdoor\ on \csp s of arity bounded by $r$, one only has to prove value-renamability, domain-decomposability (which is usually straightforward) and exhibit a language $\Gamma$ such that each $\{R\} \subset \Gamma$ is in $T$ but $\Gamma$ is not.

\begin{example}
An idempotent operation $f$ is totally symmetric (TSI) if it satisfies $f(x_1,\ldots,x_a) = f(y_1,\ldots,y_a)$ whenever $\{x_1,\ldots,x_a\} = \{y_1,\ldots,y_a\}$. \adding{Using the same notations as in Example \ref{exm:bw}}, it has been shown in~\cite{dalmau1999closure} that \changing{\csp($\overline{\Gamma}$)} is solved by arc-consistency if and only if $\Gamma$ has TSI polymorphisms of all arities. We show that this class of languages (which we denote by $T_{\text{TSI}}$) falls in the scope of Theorem~\ref{thm:nothelly} even for binary relations. First, this class is composite and idempotent: If we denote by TS($a$) the set of all possible TSI operations on $\mathbb{N}$ of arity $a$ and ATS = $\prod_{a \in \mathbb{N}^*}TS(a)$, we have
$T_{\text{TSI}} = \cup_{\mathcal{F} \in \text{ATS}} \left( \cap_{f \in \mathcal{F}} T_f \right)$. To prove domain-decomposability and value-renamability, we will use the equivalent and more convenient characterization that $\Gamma$ is in $T_{\text{TSI}}$ if and only if $\Gamma$ has a TSI of arity $|D(\Gamma)|$. \adding{Without loss of generality, we consider TSI polymorphisms as set functions and write $f(x_1,\ldots,x_a) = f(\{x_1,\ldots,x_a\})$.}
\begin{itemize}
\item Domain-decomposability: Let $\Gamma_1, \Gamma_2 \in T$ be constraint languages with respective TSI polymorphisms $f_1, f_2$ (of respective arities $|D(\Gamma_1)|, |D(\Gamma_2)|$), and $D(\Gamma_1) \cap D(\Gamma_2) = \emptyset$. Let $f$ be the operation on $D(\Gamma_1) \cup D(\Gamma_2)$ of arity $|D(\Gamma_1) \cup D(\Gamma_2)|$ defined as follows:
$$
f(x_1,\ldots,x_m)=\left\{
\begin{array}{c l}     
    f_1(\{x_1,\ldots,x_m\}) & \text{if } \{x_1,\ldots,x_m\} \subseteq D(\Gamma_1)\\
    f_2(\{x_1,\ldots,x_m\}) & \text{if } \{x_1,\ldots,x_m\} \subseteq D(\Gamma_2)\\
    \max(x_1,\ldots,x_m) & \text{otherwise}
\end{array}\right.
$$
$f$ is totally symmetric and preserves both $\Gamma_1$ and $\Gamma_2$, so $f$ is a polymorphism of $\Gamma_1 \cup \Gamma_2$ and $\Gamma_1 \cup \Gamma_2 \in T_{\text{TSI}}$. Therefore, $T_{\text{TSI}}$ is domain-decomposable.
\item Value-renamability: Let $\Gamma \in T_{\text{TSI}}$ be a language with a TSI polymorphism $f_1$ of arity $|D(\Gamma)|$. Let $\phi: D_1 \rightarrow D_2$ be a bijection, where $D_1$ is a conservative subdomain of $D(\Gamma)$ and $D_2 \cap D(\Gamma) = \emptyset$. Then, the operation of arity $|D(\Gamma) \cup D_2|$ defined as
$$
f(x_1,\ldots,x_m)=\left\{
\begin{array}{c l}     
    f_1(\{x_1,\ldots,x_m\}) & \text{if } \{x_1,\ldots,x_m\} \subseteq D(\Gamma)\\
    \phi(f_2(\{\phi^{-1}(x_1),\ldots,\phi^{-1}(x_m)\})) & \text{if } \{x_1,\ldots,x_m\} \subseteq D_2\\
    \max(x_1,\ldots,x_m) & \text{otherwise}
\end{array}\right.
$$
is a TSI and preserves both $\Gamma$ and $R_{\phi}$ (the proof is straightforward using the fact that $D_1$ is a conservative subdomain), so $\Gamma \cup \{R_\phi\} \in T_{\text{TSI}}$ and $T_{\text{TSI}}$ is value-renamable.
\item Not 1-Helly: Let $R_1 = [(0,0),(0,1),(1,0)]$ and $R_2 = [(1,1),(0,1),(1,0)]$. Both $\{R_1\}$ and $\{R_2\}$ are in $T_{\text{TSI}}$ (as they are respectively closed by min and max, which are 2-ary TSIs), but $\{R_1,R_2\}$ is not ($R_1$ forces $f(0,1) = 0$ and $R_2$ forces $f(0,1) = 1$ for every TSI polymorphism $f$), so $T_{\text{TSI}}$ is not $1$-Helly.
\end{itemize}
Finally, we conclude that Theorem \ref{thm:nothelly} applies to $T_{\text{TSI}}$ even for binary constraints. The same reasoning also applies to many other tractable \changing{classes}, such as languages preserved by a near-unanimity ($f(y,x,\ldots,x) = f(x,y,x,\ldots,x) = \ldots = f(x,\ldots,x,y) = x$) or a Mal'tsev ($f(x,x,y) = f(y,x,x) = y$) polymorphism.
\end{example}

\section{Related work}

A very recent paper by Gaspers et al.~\cite{Gaspers} has \adding{independently} investigated the same topic (parameterized complexity of strong backdoor detection for tractable semantic classes) and some of their results seem close to ours. In particular, one of their theorems (Theorem 5) is similar to our Proposition \ref{prp:w2bool}, but is less general as they assume the target class to be an union of atomic classes. They also \changing{study} the case where $r$ is bounded and $k$ is the parameter, as we do, but their result (Theorem 6) is more specific and can be shown to be implied by our Theorem \ref{thm:nothelly}.\

\section{Conclusion}

We have shown that finding small strong backdoors to tractable constraint languages is often hard. In particular, if the tractable class is a set of languages closed by an idempotent operation, or can be defined by arbitrary conjunctions and disjunctions of such languages, then finding a backdoor to this class is NP-hard even when all constraints have a fixed arity. Moreover, it is W[2]-hard with respect to the backdoor size $k$.

When considering the larger parameter $k+r$, however, we have shown that strong backdoor detection is FPT provided that the target class 
is $h$-Helly for a constant $h$, that is, membership in this class can be decided by checking all $h$-tuples of relations. We then give a complete characterization of $1$-Helly classes, and we use this result to show that any finite union of $1$-Helly classes induces a backdoor problem FPT in $k+r$.
%has a constant Helly number.
%
Finally, we characterize another large family of tractable classes for which backdoor detection is W[2]-hard for the parameter $k$ even if $r$ is fixed. This result can be used to derive hardness of backdoor detection for many known large tractable classes, provided they have certain natural properties (which we call value-renamability and domain-decomposability).

\newpage

\bibliography{bibfile}
\bibliographystyle{splncs}

\end{document}